\documentclass[11pt,a4paper]{article}

\usepackage{amsmath,amsfonts,amscd,amsthm}
\usepackage{epsfig,psfrag}
\usepackage{latexsym}
\usepackage{dsfont,enumerate,exscale}
\usepackage[boxruled]{algorithm2e}

\newcommand{\bmat}{\begin{pmatrix}}
\newcommand{\emat}{\end{pmatrix}}

\newcommand{\N}{\mathbb{N}}
\newcommand{\R}{\mathbb{R}}
\newcommand{\Order}{\mathcal{O}}
\newcommand{\Iup}{I_{\text{up}}}
\newcommand{\Idown}{I_{\text{down}}}
\newcommand{\Feas}{\mathcal{R}}

\newcommand{\Opt}{\mathcal{R^*}}

\newcommand{\argmax}{\arg \max}
\newcommand{\sign}{\operatorname{sign}}

\newcommand{\GI}{\R^{\geq 0} \cup \{\infty\}}
\newcommand{\B}{\mathcal{B}}

\newcommand{\ds}{\texttt}
\newcommand{\dss}{\small \texttt}
\newcommand{\tabfont}{\small}

\newtheorem{definition}{Definition}
\newtheorem{lemma}{Lemma}
\newtheorem{theorem}{Theorem}
\newtheorem{corollary}{Corollary}

\begin{document}

\title{The Planning-ahead SMO Algorithm}

\author{Tobias Glasmachers \\
		\texttt{tobias.glasmachers@ini.rub.de} \\
		Optimization of Adaptive Systems Group\\
		Institut f\"ur Neuroinformatik\\
		Ruhr-Universit\"at Bochum\\
		Germany}

\maketitle

\begin{abstract}%
The sequential minimal optimization (SMO) algorithm and variants thereof
are the de facto standard method for solving large quadratic programs
for support vector machine (SVM) training. In this paper we propose a
simple yet powerful modification. The main emphasis is on an algorithm
improving the SMO step size by planning-ahead. The theoretical analysis
ensures its convergence to the optimum. Experiments involving a large
number of datasets were carried out to demonstrate the superiority of
the new algorithm.
\end{abstract}


\section{Introduction}

Training a support vector machine (SVM) for binary classification is
usually accomplished through solving a quadratic program. Assume we are
given a training dataset
$(x_1, y_1), \dots, (x_{\ell}, y_{\ell})$
composed of inputs $x_i \in X$ and binary labels $y_i \in \{\pm1\}$,
a positive semi-definite Mercer kernel function $k:X \times X \to \R$ on
the input space $X$ and a regularization parameter value $C > 0$. Then
the dual SVM training problem is given by
\begin{align}
	\text{maximize} \quad & f(\alpha) = y^T \alpha - \frac{1}{2} \alpha^T K \alpha \label{eq:QP} \\
	\text{s.t.} \quad & \sum_{i=1}^{\ell} \alpha_i = 0 \hspace{25mm} \tag{equality constraint} \\
	\text{and} \quad & L_i \leq \alpha_i \leq U_i \enspace\enspace \forall \enspace 1 \leq i \leq \ell \qquad \tag{box constraint}
\end{align}
for $\alpha \in \R^{\ell}$. Here, the vector
$y = (y_1, \dots, y_{\ell})^T \in \R^{\ell}$ is composed of the labels,
the positive semi-definite kernel Gram matrix
$K \in \R^{\ell \times \ell}$ is given by $K_{ij} = k(x_i, x_j)$ and the
lower and upper bounds are $L_i = \min\{0, y_i C\}$ and
$U_i = \max\{0, y_i C\}$. We denote the feasible region by
$\Feas$ and the set of optimal points by $\Opt$.
On $\Opt$ the objective function $f$ attains its maximum denoted by
$f^* = \max \{f(\alpha) \,|\, \alpha \in \Feas \}$.

Solving this problem up to a sufficient accuracy seems to scale roughly
quadratic in $\ell$ in practice~\cite{joachims:99}. This relatively bad
scaling behavior is one of the major drawbacks of SVMs in general as the
number $\ell$ of training examples may easily range to hundreds of
thousands or even millions in today's pattern recognition problems.

\section{State of the Art SMO Algorithm}

The sequential minimal optimization (SMO) algorithm~\cite{platt:99} is
an iterative decomposition algorithm~\cite{osuna:97} using minimal
working sets of size two.
This size is minimal to keep the current solution feasible. The
algorithm explicitly exploits the special structure of the constraints
of problem~\eqref{eq:QP} and shows very good performance in practice.
For each feasible point $\alpha \in \Feas$ we define the index sets
\begin{align*}
	\Iup(\alpha) = & \{ i \in \{1, \dots, \ell\} \enspace|\enspace \alpha_i < U_i \} \\
	\Idown(\alpha) = & \{ i \in \{1, \dots, \ell\} \enspace|\enspace \alpha_i > L_i \}
	\enspace.
\end{align*}
The canonical form of the SMO algorithm (using the common
Karush-Kuhn-Tucker (KKT) violation stopping condition)
can be stated as follows:

\begin{algorithm}
	\SetKwRepeat{InfiniteLoop}{do}{loop}
	\KwIn{feasible initial point $\alpha^{(0)}$, accuracy $\varepsilon \geq 0$}
	compute the initial gradient $G^{(0)} \leftarrow \nabla f(\alpha^{(0)}) = y - K \alpha^{(0)}$

	set $t \leftarrow 1$

	\InfiniteLoop{}
	{
		\lnl{SMO-step-select-ws} select a working set $B^{(t)}$

		\lnl{SMO-step-solve} solve the sub-problem induced by $B^{(t)}$ and $\alpha^{(t-1)}$, resulting in $\alpha^{(t)}$

		\lnl{SMO-step-gradient-update} compute the gradient $G^{(t)} \leftarrow \nabla f(\alpha^{(t)}) = G^{(t-1)} - K \left( \alpha^{(t)} - \alpha^{(t-1)} \right)$

		\lnl{SMO-step-stop} stop if $\left( \max \big\{ G^{(t)}_i\ \,\big|\, i \in \Iup(\alpha^{(t)}) \big\} - \min \big\{ G^{(t)}_j \,\big|\, j \in \Idown(\alpha^{(t)}) \big\} \right) \leq \varepsilon$

		set $t \leftarrow t + 1$
	}
	\caption{General SMO Algorithm}
	\label{algo:SMO}
\end{algorithm}

If no additional information are available the initial solution is
chosen to be $\alpha^{(0)} = (0, \dots, 0)^T$ resulting in the initial
gradient $G^{(0)} = \nabla f(\alpha^{(0)}) = y$ which can be computed
without any kernel evaluations.

It is widely agreed that the working set selection policy is crucial for
the overall performance of the algorithm. This is because starting from
the initial solution the SMO algorithm generates a sequence
$(\alpha^{(t)})_{t \in \N}$ of solutions which is determined by the
sequence of working sets $(B^{(t)})_{t \in \N}$.
We will briefly discuss some concrete working set selection policies
later on.

First we will fix our notation.
In each iteration the algorithm selects a working set of size two. In
this work we will consider (ordered) tuples instead of sets for a number
of reasons. Of course we want our tuples to correspond to sets of
cardinality two. Therefore a working set $B$ is of the form $(i, j)$
with $i \not= j$. Due to its wide spread we will stick to the term
working {\em set} instead of tuple as long as there is no ambiguity.
Whenever we need to refer to the corresponding set, we will use the
notation $\widehat{B} = \widehat{(i, j)} := \{i, j\}$.
For a tuple $B = (i, j)$ we define the direction $v_B = e_i - e_j$ where
$e_n$ is the $n$-th unit vector of $\R^{\ell}$. This direction has a
positive component for $\alpha_i$ and a negative component for
$\alpha_j$. We will restrict the possible choices such that the current
point $\alpha$ can be moved in the corresponding direction $v_B$ without
immediately leaving the feasible region. This is equivalent to
restricting $i$ to $\Iup(\alpha)$ and $j$ to $\Idown(\alpha)$. We
collect the allowed working sets in a point $\alpha$ in the set
$\B(\alpha) = \Iup(\alpha) \times \Idown(\alpha) \setminus \{(n, n) \,|\, 1 \leq n \leq \ell\}$.
With this notation a working set selection policy returns some
$B^{(t)} \in \B(\alpha^{(t-1)})$.

The sub-problem induced by the working set $B^{(t)}$ solved in
step~\ref{SMO-step-solve} in iteration~$t$ is defined as
\begin{align*}
	\text{maximize} \quad & f(\alpha^{(t)}) = y^T \alpha^{(t)} - \frac{1}{2} (\alpha^{(t)})^T K \alpha^{(t)} \\
	\text{s.t.} \quad & \sum_{i=1}^{\ell} \alpha^{(t)}_i = 0 \hspace{26mm} & \text{(equality constraint)} \\
	\text{} \quad & L_i \leq \alpha^{(t)}_i \leq U_i \text{ for } i \in \widehat{B}^{(t)} \qquad & \text{(box constraint)} \\
	\text{and} \quad & \alpha^{(t)}_i = \alpha^{(t-1)}_i \text{ for } i \not\in \widehat{B}^{(t)}
	\enspace.
\end{align*}
That is, we solve the quadratic program as good as possible while
keeping all variables outside the current working set constant.
We can incorporate the equality constraint into the parameterization
$\alpha^{(t)} = \alpha^{(t-1)} + \mu^{(t)} v_{B^{(t)}}$ and arrive at
the equivalent problem
\begin{align*}
	\text{maximize} \quad & l_t \mu^{(t)} - \frac{1}{2} Q_{tt} (\mu^{(t)})^2 \\
	\text{s.t.}     \quad & \tilde L_t \leq \mu^{(t)} \leq \tilde U_t
\end{align*}
for $\mu^{(t)} \in \R$ with
\begin{align*}
	Q_{tt} = & K_{ii} - 2 K_{ij} + K_{jj} = v_{B^{(t)}}^T K v_{B^{(t)}} \\
	l_t = & \frac{\partial f}{\partial \alpha_{i}}(\alpha^{(t-1)}) - \frac{\partial f}{\partial \alpha_{j}}(\alpha^{(t-1)}) = v_{B^{(t)}}^T \nabla f(\alpha^{(t-1)}) \\
	\tilde L_t = & \max \{ L_i - \alpha_i^{(t-1)}, \alpha_j^{(t-1)} - U_j \} \\
	\tilde U_t = & \min \{ U_i - \alpha_i^{(t-1)}, \alpha_j^{(t-1)} - L_j \}
\end{align*}
and the notation $B^{(t)} = (i, j)$.
This problem is solved by clipping the Newton step $\mu^* = l_t/Q_{tt}$
to the bounds:
\begin{align}
	\mu^{(t)} = \max \left\{ \min \left\{ \frac{l_t}{Q_{tt}}, \tilde U_t \right\}, \tilde L_t \right\} \label{eq:clipped-newton}
	\enspace.
\end{align}
For $\mu^{(t)} = l_t / Q_{tt}$ we call the SMO step free. In this case
the SMO step coincides with the Newton step in direction $v_{B^{(t)}}$.
Otherwise the step is said to hit the box boundary.

Recently it has been observed that the SMO step itself can be
used for working set selection resulting in so called second order
algorithms~\cite{fan:2005,glasmachers:2006}.
We can formally define the gain of a SMO step as the function
$g_B(\alpha)$ which computes the difference $f(\alpha') - f(\alpha)$ of
the objective function before and after a SMO step with starting point
$\alpha$ on the working set $B$, resulting in $\alpha'$. For each
working set $B$ this function is continuous and piecewise quadratic
(see \cite{glasmachers:2006}). Then these algorithms greedily choose a
working set~$B^{(t)}$ promising the largest functional gain
$g_{B^{(t)}}(\alpha^{(t-1)}) = f(\alpha^{(t)}) - f(\alpha^{(t-1)})$ by
heuristically evaluating a subset of size $\Order(\ell)$ of the
possible working sets $\B(\alpha^{(t-1)})$.

Fan et al.~\cite{fan:2005} propose to choose the working set according to
\begin{align}
	i = \, & \argmax \left\{ \left. \frac{\partial f}{\partial \alpha_n}(\alpha) \enspace\right|\enspace n \in \Iup(\alpha) \right\} \notag \\
	j = \, & \argmax \Big\{ \tilde g_{(i, n)}(\alpha) \enspace\Big|\enspace n \in \Idown(\alpha) \setminus \{i\} \Big\} \label{eq:wss} \\
	\text{with} \quad & \tilde g_B(\alpha) = \frac{1}{2} \frac{(v_B^T \nabla f(\alpha))^2}{v_B^T K v_B} \in \GI \notag
\end{align}
where $\tilde g_B(\alpha)$ is an upper bound on the gain which is exact
if and only if the step starting from $\alpha$ with working set $B$ is
not constrained by the box.%
\footnote{The software LIBSVM~\cite{fan:2005} sets the denominator of
$\tilde g_B(\alpha)$ to $\tau = 10^{-12} > 0$ whenever it vanishes.
This way the infinite value is avoided. However, this trick was
originally designed to tackle indefinite problems.}
Note that in this case the Newton
step~$\mu^* = (v_B^T \nabla f(\alpha)) / (v_B^T K v_B)$
in direction $v_B$ is finite and we get the alternative formulation
\begin{align}
	\tilde g_B(\alpha) = \frac{1}{2} (v_B^T K v_B) (\mu^*)^2 \label{eq:newton-gain}
	\enspace.
\end{align}
This formula can be used to explicitly compute the exact SMO gain
$g_B(\alpha)$ by plugging in the clipped step
size~\eqref{eq:clipped-newton} instead of the Newton step~$\mu^*$.

The stopping condition in step~\ref{SMO-step-stop} checks if the
Karush-Kuhn-Tucker (KKT) conditions of problem~\eqref{eq:QP} are
fulfilled with the predefined accuracy~$\varepsilon$. List et al.~\cite{list:2007}
have shown that this is a meaningful stopping criterion. The accuracy
$\varepsilon$ is usually set to $0.001$ in practice.

SMO is a specialized version of the more general decomposition algorithm
which imposes the weaker condition $|B^{(t)}| \leq q \ll \ell$ on the
working set size. The main motivation for decomposition is that in each
step only the rows of the kernel matrix $K$ which correspond to the
working set indices are needed. Therefore the algorithm works well even
if the whole matrix $K$ does not fit into the available working memory.
The SMO algorithm has the advantage over decomposition with larger
working sets that the sub-problems in step~\ref{SMO-step-solve} can be
solved very easily. Because of its minimal working set size the
algorithm makes less progress in a single iteration compared to larger
working sets. On the other hand single iterations are faster. Thus,
there is a trade-off between the time per iteration and the number of
iterations needed to come close enough to the optimum. The decisive
advantage of SMO in this context is that it can take its decisions which
working set~$B$ (corresponding to the optimization direction~$v_B$) to
choose more frequently between its very fast iterations. This strategy
has proven beneficial in practice.

In elaborate implementations
the algorithm is accompanied by a kernel cache and a shrinking
heuristic~\cite{joachims:99}. The caching technique exploits the fact
that the SMO algorithm needs the rows of the kernel matrix which
correspond to the indices in the current working set $B^{(t)}$. The
kernel cache uses a predefined amount of working memory to store rows of
the kernel matrix which have already been computed. Therefore the
algorithm needs to recompute only those rows from the training data
evaluating the possibly costly kernel function which have not been used
recently. The shrinking heuristic removes examples from the problem that
are likely to end up at the box boundaries in the final solution.
These techniques perfectly cooperate and result in an enormous speed up
of the training process. We will later use the fact that the most
recently used rows of the kernel matrix $K$ are available from the
cache.

The steps~\ref{SMO-step-select-ws}, \ref{SMO-step-gradient-update}, and
\ref{SMO-step-stop} of the SMO optimization loop take $\Order(\ell)$
operations, while the update~\ref{SMO-step-solve} of the current
solution is done in constant time.

There has not been any work on the improvement of
step~\ref{SMO-step-solve} of Algorithm~\ref{algo:SMO}.
Of course, it is not possible to considerably speed up a computation
taking $\Order(1)$ operations, but we will see in the following how we
may replace the optimal (greedy) truncated Newton step with other
approaches.

\section{Behavior of the SMO Algorithm} \label{sec:behavior}

We want to make several empirical and theoretical statements about the
overall behavior of the SMO algorithm. This includes some motivation for
the algorithm presented later. We start with theoretical results.

It is well known that the algorithm converges to an optimum for a number
of working set selection strategies. Besides convergence proofs for
important special cases~\cite{keerthi-gilbert:2002,takahashi:2005,glasmachers:2006,glasmachers:esann08}
proof techniques for general classes of selection policies have been
investigated~\cite{hush:2003,list:2004,chen:2006}.

Chen et al.~\cite{chen:2006} have shown that under some technical conditions on
problem~\eqref{eq:QP} there exists $t_0$ such that no SMO step ends up
at the box bounds for iterations $t > t_0$. For these iterations the
authors derive a linear convergence rate.
However, the prerequisites exclude the relevant case that the optimum is
not isolated. Upper bounds for $t_0$ are not known, and in experiments
the algorithm rarely seems to enter this stage.

From an empirical point of view we can describe the qualitative behavior
of SMO roughly as follows. In the first iterations, starting from the
initial solution $\alpha^{(0)} = (0, \dots, 0)^T$, many steps move
variables $\alpha_i$ to the lower or upper bounds $L_i$ and $U_i$. After
a while, these steps become rare and most iterations are spent on a
relatively small number of variables performing free steps. In this
phase the shrinking heuristics removes most bounded variables from the
problem. Then working set selection, gradient update and stopping
condition need to be computed only on the relatively small set of active
variables, leading to extremely fast iterations.

Many common benchmark problems and real world applications are simple in
the sense that the algorithm performs only a number of iterations
comparable to the number of examples. In this case only very few
variables (if any) are changed many times. This indicates that there are
only very few free support vectors or that the dependencies between
variables are weak, making the optimization easy. On the other hand, for
harder problems the algorithm spends most of its iterations on free SMO
steps to resolve complicated dependencies between the free variables.
In fact, in some cases we can observe large blocks of iterations spent
on a small number of variables.
Due to its finite number of optimization directions the SMO algorithm is
prone to oscillate while compensating the second order cross terms of
the objective function. This oscillatory behavior observed in case of
difficult problems is the main motivation for the consideration
presented in the next section.

\section{Planning Ahead} \label{sec:planning-ahead}

Without loss of generality we consider the iteration $t=1$ in this
section.

Assume we are given the current working set
$B^{(1)} = (i^{(1)}, j^{(1)})$ and for some reason we already know the
working set $B^{(2)} = (i^{(2)}, j^{(2)})$ to be selected in
the next iteration.
In addition, we presume that the solutions of both sub-problems involved
are not at the bounds. That is, we can simply ignore the box constraints.
From equation~\eqref{eq:newton-gain} we know with
$\mu^{(1)} = l_1 / Q_{11}$ and $\mu^{(2)} = l_2 / Q_{22}$
that both free steps together result in the gain
\begin{align} \label{eq:gain1}
	g^{\text{2-step}} := f(\alpha^{(2)}) - f(\alpha^{(0)}) = \frac{1}{2} Q_{11} (\mu^{(1)})^2 + \frac{1}{2} Q_{22} (\mu^{(2)})^2
	\enspace.
\end{align}
Under the assumption that we already know the working set $B^{(2)}$ we
can of course precompute the second step. To stress this point of view
we introduce the quantities
\begin{align*}
	w_t = \frac{\partial f}{\partial \alpha_{i^{(t)}}}(\alpha^{(0)}) - \frac{\partial f}{\partial \alpha_{j^{(t)}}}(\alpha^{(0)}) = v_{B^{(t)}} \nabla f(\alpha^{(0)}) \qquad \text{for } t \in \{1, 2\}
\end{align*}
which only depend on $\alpha^{(0)}$ and are thus known in iteration
$t=1$. We rewrite
\begin{align*}
						 l_1 & = w_1 \\
						 l_2 & = w_2 - Q_{12} \mu^{(1)} \\
	\text{with} \quad Q_{12} & = Q_{21} = K_{i^{(1)}i^{(2)}} - K_{i^{(1)}j^{(2)}} - K_{j^{(1)}i^{(2)}} + K_{j^{(1)}j^{(2)}} = v_{B^{(1)}}^T K v_{B^{(2)}}
	\enspace.
\end{align*}
Then we can express the step size
\begin{align}
	\mu^{(2)} = l_2 / Q_{22} = w_2 / Q_{22} - Q_{12} / Q_{22} \mu^{(1)} \label{eq:mu2_mu1}
\end{align}
in these terms.
The above notation suggests the introduction of the $2 \times 2$ matrix
\begin{align*}
	Q = \bmat Q_{11} & Q_{21} \\ Q_{12} & Q_{22} \emat
\end{align*}
which is symmetric and positive semi-definite.
If we drop the assumption that both steps involved are Newton steps
the computation of the gain is more complicated:
\begin{align*}
	g^{\text{2-step}}(\mu^{(1)}, \mu^{(2)})
		:= f(\alpha^{(2)}) - f(\alpha^{(0)})
		= \bmat w_1 \\ w_2 \emat^T \bmat \mu^{(1)} \\ \mu^{(2)} \emat - \frac{1}{2} \bmat \mu^{(1)} \\ \mu^{(2)} \emat^T Q \bmat \mu^{(1)} \\ \mu^{(2)} \emat
\end{align*}
Plugging everything in, and in particular substituting $\mu^{(2)}$
according to eq.~\eqref{eq:mu2_mu1} we express the gain as a function of
the single variable $\mu^{(1)}$ resulting in
\begin{align} \label{eq:gain2}
	g^{\text{2-step}}(\mu^{(1)})
		= - \frac{1}{2} \cdot \frac{\det(Q)}{Q_{22}} (\mu^{(1)})^2
		+ \frac{Q_{22} w_1 - Q_{12} w_2}{Q_{22}} \mu^{(1)}
		+ \frac{1}{2} \cdot \frac{w_2^2}{Q_{22}}
	\enspace.
\end{align}
For $\mu^{(1)} = w_1 / Q_{11}$ we obtain the gain computed
in~\eqref{eq:gain1}, but of course the maximizer of the quadratic
function~\eqref{eq:gain2} will in general differ from this value. Thus,
under the above assumption that we already know the next working set
$B^{(2)}$ we can achieve a better functional gain by computing the
optimal step size
\begin{align}
	\mu^{(1)} = \frac{Q_{22} w_1 - Q_{12} w_2}{\det(Q)} \label{eq:planning-step}
\end{align}
where we again assume that we do not hit the box constraints which are
dropped.
It is easy to incorporate the constraints into the computation, but this
has two drawbacks in our situation: First, it leads to a large number of
different cases, and second it complicates the convergence proof.
Further, dropping the constraints will turn out to be no restriction, as
the algorithms resulting from these considerations will handle the box
constrained case separately. Figure~\ref{fig:TwoStep} illustrates the
resulting step.
We call the step $\mu^{(1)} \cdot v_{B^{(1)}}$ the planning-ahead step,
because we need to simulate the current and the next step in order to
determine the step size~$\mu^{(1)}$. Analogously we refer to~$\mu^{(1)}$
as the planning-ahead step size.

\begin{figure}[h]
\begin{center}
\psfrag{a}[c][l]{$v_{B^{(2)}}$}
\psfrag{b}[c][l]{$v_{B^{(1)}}$}
\includegraphics{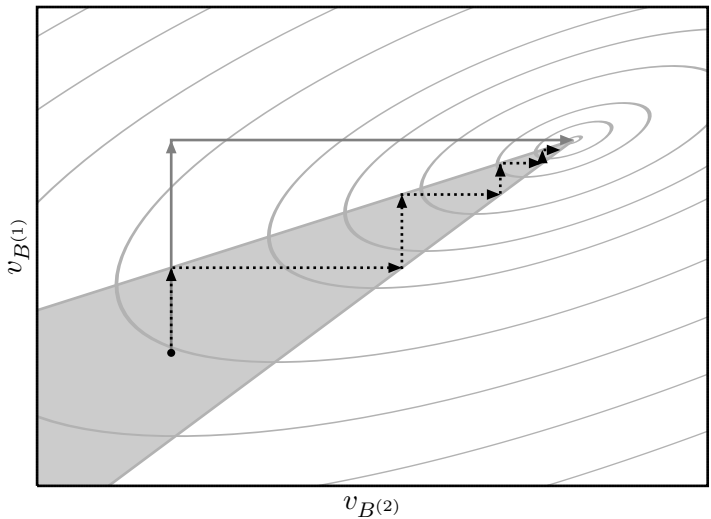}
\caption{\label{fig:TwoStep}
Optimization path with (dark gray) and without (dotted) planning ahead
in a minimal scenario composed of only two possible working sets
$B^{(1)}$ and $B^{(2)}$. The light gray ellipses indicate niveau sets of
the objective function~$f$. From these it is obvious that the first step
of the dark gray path results in a decrease of the objective function.
However, while the usual SMO procedure oscillates inside a cone bounded
by the hyperplanes $\{\alpha \,|\, v_{B^{(t)}}^T \nabla f(\alpha) = 0 \}$,
planning ahead one step results in the optimal step size to solve this
low dimensional problem.}
\end{center}
\end{figure}

Just like the usual SMO update this step can be computed in constant
time, that is, independent of the problem dimension~$\ell$. However, the
kernel values of an up to $4 \times 4$ principal minor of the kernel
Gram matrix $K$ are needed for the computation, in contrast to a
$2 \times 2$ minor for the standard SMO update step.

Note the asymmetry of the functional gain as well as the optimal step
size w.r.t.\ the iteration indices $1$ and $2$. We control the length of
the first step, which of course influences the length of the second
step. The asymmetry results from the fact that the second step is greedy
in contrast to the first one.
The first step is optimal given the next working set $B^{(2)}$ and
planning one step ahead, while the second step is optimal in the usual
sense of doing a single greedy step without any planning-ahead.

\begin{figure}[h]
\begin{center}
\psfrag{x}[l][l]{$\frac{\mu}{\mu^*}$}
\psfrag{y}[c][l]{$\tilde g$}
\psfrag{0}[c][l]{$0$}
\psfrag{1-e}[l][l]{$1 - \eta$}
\psfrag{1}[c][l]{$1$}
\psfrag{1+e}[r][l]{$1 + \eta$}
\psfrag{2}[c][l]{$2$}
\psfrag{g}[r][l]{$\tilde g^*$}
\psfrag{rg}[r][l]{$(1 - \eta^2) \tilde g^*$}
\psfrag{f}[l][l]{$\tilde g(\mu) = \left( 2 \frac{\mu}{\mu^*} - (\frac{\mu}{\mu^*})^2 \right) \tilde g^*$}
\includegraphics{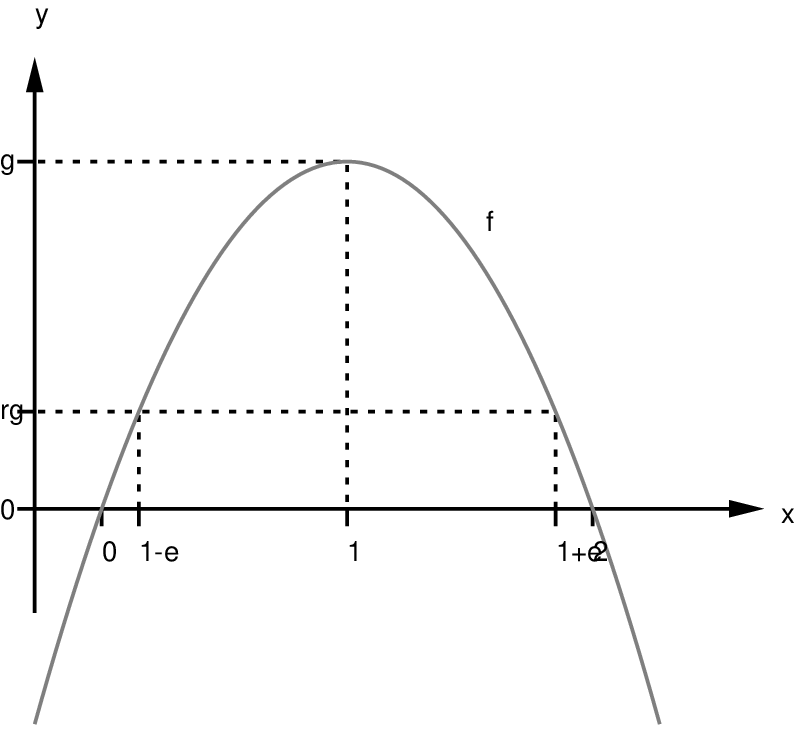}
\caption{\label{fig:parabola1}
Gain of the step with size $\mu$ compared to the Newton step size
$\mu^*$. As long as the quantity $\mu / \mu^*$ is in the open interval
$(0, 2)$ the step makes some progress. For the interval
$[1-\eta, 1+\eta]$ the gain is even strictly lower bounded by a fraction
of $1-\eta^2$ of the Newton step gain $\tilde g^*$.
Note that if the working set direction $v_B$ is in the kernel of the
matrix $K$ the parabola is degenerated to a linear function. Then we
distinguish two cases: If the linear term of $\tilde g$ vanishes we set
$\mu^* = 0$ and get $\tilde g^* = 0$. Otherwise the graph grows linearly
corresponding to $\mu^* = \pm\infty$ and $\tilde g^* = \infty$.}
\end{center}
\end{figure}

Another interesting property is that for
$\mu^{(1)} \not\in [0, 2 \, l_1 / Q_{11}]$ the first step actually
results in a decay of the dual objective, that is,
$f(\alpha^{(1)}) < f(\alpha^{(0)})$, see Figure~\ref{fig:parabola1}.
Nevertheless, such steps can be
extremely beneficial in some situations, see Figure~\ref{fig:TwoStep}.
Of course, by construction both planned-ahead steps together result in
an increase of the objective function, which is even maximized for the
given working sets.

\section{Algorithms} \label{sec:algorithms}

In this section we will turn the above consideration into algorithms.
We will present a first simple version and a refinement which focuses
on the convergence of the overall algorithm to an optimal solution.
These modifications are all based on the SMO~Algorithm~\ref{algo:SMO}.
Thus, we will only state replacements for the working set selection
step~\ref{SMO-step-select-ws} and the update step~\ref{SMO-step-solve}.

In the previous section it is left open how we can know the choice of
the working set in the forthcoming iteration. If we try to compute this
working set given a step size $\mu$, it turns out that we need to run
the working set selection algorithm. That is, the next working set
depends on the current step size and it takes linear time to determine
the working set for a given step size. This makes a search for the
optimal combination of step size and working set impractical.
We propose a very simple heuristic instead. For two reasons we suggest
to reuse the previous working set: First, the chance that the
corresponding kernel evaluations are cached is highest for this working
set.
Second, as already stated in section~\ref{sec:behavior}, the SMO
algorithm sometimes tends to oscillate within a small number of
variables. Figure~\ref{fig:TwoStep} gives a low-dimensional example.
Now, if we are in a phase of oscillation, the previous working set is a
very good candidate for planning-ahead.

These considerations result in a new algorithm. It differs from the SMO
algorithm only in step~\ref{SMO-step-solve}. The basic idea is to use
the previous working set for planning-ahead. However, we revert to the
standard SMO step if the previous step was used for planning-ahead or
the planned steps are not free. This proceeding is formalized in
Algorithm~\ref{algo:modified-step}, which is a replacement for
step~\ref{SMO-step-solve} of Algorithm~\ref{algo:SMO}.

\begin{algorithm}[H]
	\eIf{previous iteration performed a SMO step (eq.~\eqref{eq:clipped-newton})}
	{
		Compute the planning-ahead step size $\mu = \frac{Q_{22} w_1 - Q_{12} w_2}{\det(Q)}$ (eq.~\eqref{eq:planning-step})\\
		$\qquad$ assuming $B^{(t-1)}$ as the next working set\\
		\eIf{the current or the planned step ends at the box boundary}
		{
			perform a SMO step (eq.~\eqref{eq:clipped-newton})
		}
		{
			perform the step of size $\mu$ as planned
		}
	}
	{
		perform a SMO step (eq.~\eqref{eq:clipped-newton})
	}
	\caption{Modification of step \ref{SMO-step-solve} of the SMO Algorithm}
	\label{algo:modified-step}
\end{algorithm}

As already indicated in section~\ref{sec:planning-ahead} the algorithm
uses planning-ahead only if both steps involved do not hit the box
boundaries. This means that we need to check the box constraints while
planning ahead, and we turn to the standard SMO algorithm whenever the
precomputed steps become infeasible. Thus, there is no need to
incorporate the box constraints into the planning-ahead step size given
in equation~\eqref{eq:planning-step}.

The algorithm works well in experiments. However, it is hard to prove
its convergence to an optimal solution for a number of reasons.
The main difficulty involved is that we can not prove the strict
increase of the objective function for the planning-ahead step, even if
we additionally consider the subsequent iteration.
Therefore we additionally replace the working set selection
step~\ref{SMO-step-select-ws} of Algorithm~\ref{algo:SMO} with
Algorithm~\ref{algo:modified-selection}.

\begin{algorithm}[H]
	\KwIn{$\eta \in (0, 1)$}
	\KwIn{$\mu^{(t-1)}$: step size of the previous iteration $t-1$}
	\KwIn{$\mu^*$: Newton step size of the previous iteration $t-1$}

	\eIf{previous step resulted from planning-ahead}
	{
		\CommentSty{// standard selection, see equation \eqref{eq:wss}}\\
		$i^{(t)} \leftarrow \argmax\{\frac{\partial f}{\partial \alpha_n}(\alpha^{(t-1)}) \enspace|\enspace n \in \Iup(\alpha^{(t-1)})\}$\\
		$j^{(t)} \leftarrow \argmax\{\tilde g_{(i^{(t)}, n)}(\alpha^{(t-1)}) \enspace|\enspace n \in \Idown(\alpha^{(t-1)}) \setminus \{i\} \}$\\
		$B^{(t)} \leftarrow (i^{(t)}, j^{(t)})$
	}
	{
		\eIf{$1-\eta \leq \mu^{(t-1)} / \mu^* \leq 1+\eta$}
		{
			\CommentSty{// selection with additional candidate $B^{(t-2)}$}\\
			$i^{(t)} \leftarrow \argmax\{\frac{\partial f}{\partial \alpha_n}(\alpha^{(t-1)}) \enspace|\enspace n \in \Iup(\alpha^{(t-1)})\}$\\
			$j^{(t)} \leftarrow \argmax\{\tilde g_{(i^{(t)}, n)}(\alpha^{(t-1)}) \enspace|\enspace n \in \Idown(\alpha^{(t-1)}) \setminus \{i\} \}$\\
			$B^{(t)} \leftarrow (i^{(t)}, j^{(t)})$\\
			\If{$\tilde g_{B^{(t-2)}}(\alpha^{(t-1)}) > \tilde g_{B^{(t)}}(\alpha^{(t-1)})$}
			{
				$B^{(t)} \leftarrow B^{(t-2)}$
			}
		}
		{
			\CommentSty{// selection with additional candidate $B^{(t-2)}$}\\
			\CommentSty{// based on $g$ instead of $\tilde g$}\\
			$i^{(t)} \leftarrow \argmax\{\frac{\partial f}{\partial \alpha_n}(\alpha^{(t-1)}) \enspace|\enspace n \in \Iup(\alpha^{(t-1)})\}$\\
			$j^{(t)} \leftarrow \argmax\{g_{(i^{(t)}, n)}(\alpha^{(t-1)}) \enspace|\enspace n \in \Idown(\alpha^{(t-1)}) \setminus \{i\} \}$\\
			$B^{(t)} \leftarrow (i^{(t)}, j^{(t)})$\\
			\If{$g_{B^{(t-2)}}(\alpha^{(t-1)}) > g_{B^{(t)}}(\alpha^{(t-1)})$}
			{
				$B^{(t)} \leftarrow B^{(t-2)}$
			}
		}
	}
	\caption{Modification of step~\ref{SMO-step-select-ws} of the SMO Algorithm}
	\label{algo:modified-selection}
\end{algorithm}

At a first glance this algorithm looks more complicated than it is.
The selection basically ensures that the planning-ahead step and the
next SMO step together have a positive gain:
Recall that for $\mu^{(t-1)} / \mu^* \in [1-\eta, 1+\eta]$ the planning
step itself makes some progress, see Figure~\ref{fig:parabola1}. The
following SMO step has always positive gain. Now consider the case that
the planning-step does not make a guaranteed progress, that is,
$\mu^{(t-1)} / \mu^* \not\in [1-\eta, 1+\eta]$. The planned
double-step gain~\eqref{eq:gain2} is by construction lower bounded by
the Newton step gain. Thus, if the previous working set is reused in the
following iteration the total gain is positive. Now the usage of the SMO
gain $g$ instead of the Newton step gain $\tilde g$ for working set
selection ensures that this gain can only increase if another working
set is actually selected in the step following planning-ahead. Thus,
both steps together have positive gain in any case.
In the following we will arbitrarily fix $\eta = 0.9$. Thus, we will not
consider $\eta$ as a free hyper-parameter of
Algorithm~\ref{algo:modified-selection}.

It obviously makes sense to provide the working set which was used for
planning-ahead as a candidate to the working set selection algorithm.
As explained above, this property together with the usage of the SMO
gain function~$g$ instead of the approximation~$\tilde g$ ensures
positive gain of the double-step. Of course, positive gain is not
sufficient to show the convergence to an optimal point. The following
section is devoted to the convergence proof.

Although planning-ahead is done in constant time, it takes considerably
longer than the computation of the Newton step. For simple problems
where planning-ahead does not play a role because most steps end up at
the box bounds the unsuccessful planning steps can unnecessarily slow
down the algorithm. As discussed in section~\ref{sec:behavior} this is
mainly the case at the beginning of the optimization. We introduce the
following simple heuristic:
If the previous iteration was a free SMO step, then we perform planning
ahead, otherwise we perform another SMO step. Thus, we use the previous
SMO step as a predictor for the current one.
Algorithm~\ref{algo:predictor} captures this idea.

\begin{algorithm}[H]
	\eIf{previous iteration performed a free SMO step}
	{
		Compute the planning-ahead step size $\mu = \frac{Q_{22} w_1 - Q_{12} w_2}{\det(Q)}$ (eq.~\eqref{eq:planning-step})\\
		$\qquad$ assuming $B^{(t-1)}$ as the next working set\\
		\eIf{the current or the planned step ends at the box boundary}
		{
			perform a SMO step (eq.~\eqref{eq:clipped-newton})
		}
		{
			perform the step of size $\mu$ as planned
		}
	}
	{
		perform a SMO step (eq.~\eqref{eq:clipped-newton})
	}
	\caption{Modification of step \ref{SMO-step-solve} of the SMO Algorithm.
	The only difference compared to Algorithm~\ref{algo:modified-step}
	is the first condition that the SMO step must be {\em free}.}
	\label{algo:predictor}
\end{algorithm}

The SMO algorithm with modified steps \ref{SMO-step-select-ws} and
\ref{SMO-step-solve} as defined in Algorithms \ref{algo:predictor}
and \ref{algo:modified-selection}, respectively, will be referred to as
the planning-ahead SMO (PA-SMO) algorithm in the following. For
completeness, we state the complete PA-SMO algorithm at the end of the
paper.

It is not really clear whether the consideration of the working set
$B^{(t-1)}$ for planning ahead is a good choice. In fact, it would be
good to know whether this choice has a critical impact on the
performance. To evaluate this impact we need to introduce a variant of
PA-SMO. Because the planning-step takes only constant time we could
afford to perform $N > 1$ planning steps with the working sets
$B^{(t-n)}$ for $1 \leq n \leq N$ and choose the step size with the
largest double-step gain. In this case we should also provide these
sets to the working set selection algorithm as additional candidates.
We call this variant the multiple planning-ahead algorithm using the
$N > 1$ most recent working sets.

\section{Convergence of the Method} \label{sec:convergence}

In this section we will show that the PA-SMO algorithm converges to the
optimum $f^*$ of problem~\eqref{eq:QP}. First we introduce some notation
and make some definitions.

\begin{definition}
We say that a function  $s : \Feas \setminus \Opt \to \R$ has
property~$(*)$ if it is positive and lower semi-continuous. We extend
this definition to functions $s : \Feas \to \R$ which are positive and
lower semi-continuous on $\Feas \setminus \Opt$.
We say that a function $h$ has property~$(**)$ if there exists a
function $s$ with property~$(*)$ which is a lower bound for~$h$
on~$\Feas \setminus \Opt$.
\end{definition}

Recall two important properties of a lower semi-continuous function~$s$:
First, if $s(\alpha) > 0$ then there exists an open neighborhood $U$ of
$\alpha$ such that $s(\alpha) > 0$ for all $\alpha \in U$. Second, a
lower semi-continuous function attains its minimum on a (non-empty)
compact set.

\noindent The gap function
\begin{align*}
	\psi(\alpha) = \max \Big\{ v_B^T \nabla f(\alpha) \,\Big|\, B \in \B(\alpha) \Big\} \in \R
\end{align*}
will play a central role in the following. Note that this function is
used in the stopping condition in step~\ref{SMO-step-stop} of
Algorithm~\ref{algo:SMO} because it is positive on
$\Feas \setminus \Opt$ and zero or negative on~$\Opt$.

\begin{lemma}
The function $\psi$ has property~$(*)$.
\end{lemma}
\begin{proof}
On $\Feas$ we introduce the equivalence relation
$\alpha_1 \sim \alpha_2 \Leftrightarrow \B(\alpha_1) = \B(\alpha_2)$
and split the feasible region into equivalence classes, denoted by
$[\alpha]$.
Obviously, $\psi$ is continuous on each equivalence class $[\alpha]$.
Now, the topological boundary $\partial [\alpha]$ of a class $[\alpha]$
is the union of those classes $[\alpha']$ with
$\B([\alpha']) \subset \B([\alpha])$.
Because the argument of the maximum operation in the definition of
$\psi$ is a subset of $\B([\alpha])$ on the boundary the maximum can
only drop down. Thus, $\psi$ is lower semi-continuous. Further, it is
well known that $\psi$ is positive for non-optimal points.
\end{proof}

\noindent We collect all possible working sets in
\begin{align*}
	\B = \Big\{ (i, j) \enspace\Big|\enspace i, j \in \{1, \dots, \ell\} \text{ and } i \not= j \Big\}
\end{align*}
and write the working set selection~\eqref{eq:wss} as a map
$W : \Feas \to \B$. With this fixed working set selection we
consider the Newton step gain
\begin{align*}
	\tilde g_W : \Feas \to \GI, \qquad \tilde g_W(\alpha) = \tilde g_{W(\alpha)}(\alpha)
\end{align*}
as a function of~$\alpha$, in contrast to the family of functions with
variable working set defined in eq.~\eqref{eq:newton-gain}.

\begin{lemma} \label{lemma:bound}
There exists $\sigma > 0$ such that the function
$\varphi(\alpha) = \sigma \, (\psi(\alpha))^2$ is a lower bound
for~$\tilde g_W$ on $\Feas \setminus \Opt$. Thus, the Newton step gain
$\tilde g_W$ has property~$(**)$.
\end{lemma}
\begin{proof}
Of course, $\varphi$ inherits property~$(*)$ from~$\psi$.
We split $\Feas \setminus \Opt = M \cup N$ into disjoint subsets
\begin{align*}
	M & = \Big\{ \alpha \in \Feas \setminus \Opt \enspace\Big|\enspace v_{W(\alpha)} \not\in \ker(K) \Big\} \\
	N & = \Big\{ \alpha \in \Feas \setminus \Opt \enspace\Big|\enspace v_{W(\alpha)} \in \ker(K) \Big\}
\end{align*}
and introduce the constants
\begin{align*}
	\sigma_1 = \max \Big\{ v_B^T K v_B &\enspace\Big|\enspace B \in \B \Big\} \\
	\sigma_2 = \min \Big\{ v_B^T K v_B &\enspace\Big|\enspace B \in \B \text{ with } v_B^T K v_B > 0 \Big\}
	\enspace.
\end{align*}

On $M$ we have $v_{W(\alpha)}^T K v_{W(\alpha)} > 0$.
We proceed in two steps.
First we define the gap of the working set $W(\alpha)$
\begin{align*}
	\psi_W : \Feas \to \R, \qquad \alpha \mapsto v_{W(\alpha)}^T \nabla f(\alpha)
	\enspace.
\end{align*}
Then we make use of a result from \cite{fan:2005}, where the bound
$\psi_W(\alpha) \geq \sqrt{\sigma_2 / \sigma_1} \, \psi(\alpha)$
is derived in section~3.
From the definition of $\tilde g_B(\alpha)$ applied to the working set
$B(\alpha)$ we get the inequality
\begin{align*}
	\tilde g_W(\alpha) = \frac{1}{2} \frac{(v_{W(\alpha)}^T \nabla f(\alpha))^2}{v_{W(\alpha)}^T K v_{W(\alpha)}} \geq \frac{1}{2 \, \sigma_1} \, (\psi_W(\alpha))^2
\end{align*}
resulting in the desired lower bound with
\begin{align*}
	\sigma = \frac{1}{2 \, \sigma_1} \left( \sqrt{\frac{\sigma_2}{\sigma_1}} \right)^2 = \frac{\sigma_2}{2 \, \sigma_1^2} > 0
	\enspace.
\end{align*}

On $N$ the situation is much simpler. For
$v_{W(\alpha)}^T \nabla f(\alpha) = 0$ we can not make any progress on
the working set $W(\alpha)$ which contradicts $\alpha \not\in \Opt$.
Thus we have $v_{W(\alpha)}^T \nabla f(\alpha) \not= 0$ which implies
$\tilde g_W(\alpha) = \infty$. This is because the quadratic term of the
objective function in direction $v_{W(\alpha)}$ vanishes and the
function increases linearly. Of course we then have
$\tilde g_W(\alpha) \geq \sigma (\psi(\alpha))^2$ for $\alpha \in N$.
\end{proof}

In contrast to~\cite{fan:2005} there is no need to use an artificial
lower bound $\tau > 0$ for vanishing quadratic terms in this proof.
With the properties of $\varphi$ at hand it is straight forward to
prove the following theorem:

\begin{theorem} \label{theorem:convergence}
Consider the sequence $(\alpha^{(t)})_{t \in \N}$ in $\Feas$ with
$(f(\alpha^{(t)})_{t \in \N})$ monotonically increasing.
Let there exist a constant $c > 0$ and an infinite set $T_c \subset \N$
such that the steps from $\alpha^{(t-1)}$ to $\alpha^{(t)}$ have the
property
\begin{align*}
	f(\alpha^{(t)}) - f(\alpha^{(t-1)}) \geq c \cdot \tilde g_W(\alpha^{(t-1)})
\end{align*}
for all $t \in T_c$.
Then we have $\lim\limits_{t \to \infty} f(\alpha^{(t)}) = f^*$.
\end{theorem}
\begin{proof}
Because of the compactness of $\Feas$ there exists a convergent
sub-sequence $(\alpha^{(t-1)})_{t \in \tilde T}$ for some
$\tilde T \subset T_c$. We will denote its limit point by
$\alpha^{(\infty)}$. Assume the limit point is not optimal. Then
$\varphi(\alpha^{(\infty)}) > 0$ by property~$(*)$. The
lower semi-continuity of $\varphi$ implies $\varphi(\alpha) > 0$ for all
$\alpha$ in an open neighborhood $U'$ of $\alpha^{(\infty)}$. We choose
a smaller open neighborhood $U$ of $\alpha^{(\infty)}$ such that its
closure $\overline{U}$ is contained in $U'$. Again by lower
semi-continuity $\varphi$ attains its minimum $m > 0$ on $\overline{U}$.
There is $t_0$ such that $\alpha^{(t)} \in U$ for all $t \in \tilde T$
with $t > t_0$. Then we have
\begin{align*}
	f(\alpha^{(\infty)}) \enspace \geq & \enspace f(\alpha^{(t_0)}) + \sum_{t \in \tilde T, t > t_0} f(\alpha^{(t)}) - f(\alpha^{(t-1)}) \\
			\enspace \geq & \enspace f(\alpha^{(t_0)}) + \sum_{t \in \tilde T, t > t_0} c \cdot \tilde g(\alpha^{(t-1)}) \\
			\enspace \geq & \enspace f(\alpha^{(t_0)}) + \sum_{t \in \tilde T, t > t_0} c \cdot \varphi(\alpha^{(t-1)}) \\
			\enspace \geq & \enspace f(\alpha^{(t_0)}) + \sum_{t \in \tilde T, t > t_0} c \cdot m = \infty > f^*
\end{align*}
which is a contradiction. Thus, $\alpha^{(\infty)}$ is optimal.
\end{proof}

With the additional assumption that infinitely many SMO steps end up
free we can use Theorem~\ref{theorem:convergence} to show the
convergence of Algorithm~\ref{algo:SMO} to an optimal solution.
This was already proven in~\cite{fan:2005} without this assumption.

\begin{corollary} \label{corollary:convergence}
Consider the sequence $(\alpha^{(t)})_{t \in \N}$ in $\Feas$ with
$(f(\alpha^{(t)}))_{t \in \N}$ monotonically increasing.
If there are infinitely many $t$ such that the step from
$\alpha^{(t-1)}$ to $\alpha^{(t)}$ is a free SMO step with working
set~\eqref{eq:wss} then we have
$\lim\limits_{t \to \infty} f(\alpha^{(t)}) = f^*$.
\end{corollary}
\begin{proof}
For a free SMO step we have
$f(\alpha^{(t)}) - f(\alpha^{(t-1)}) = \tilde g(\alpha^{(t-1)})$.
Thus we can simply apply the above theorem with $c = 1$.
\end{proof}

The allurement of this approach is that we do not need any assumption on
the steps which differ from free SMO steps as long as the objective
function does not decrease. This is an ideal prerequisite to tackle the
convergence of hybrid algorithms which need to distinguish qualitatively
different branches, like for example the PA-SMO algorithm.
Consequently, the following lemma will be helpful when applying the
above results to PA-SMO.

\begin{lemma} \label{lemma:gain-rate}
Consider two iterations $t$ and $t+1$ of the PA-SMO algorithm where
planning-ahead is active in iteration $t$. The double-step gain
$g^{\text{2-step}} = f(\alpha^{(t+1)}) - f(\alpha^{(t-1)})$ is then
lower bounded by $(1 - \eta^2) \cdot \tilde g_W(\alpha^{(t-1)})$.
\end{lemma}
\begin{proof}
Let $\mu^*$ denote the Newton step size in iteration $t$ and let
$\tilde g^* = \tilde g_W(\alpha^{(t-1)})$ be the gain achieved by this
(possibly infeasible) step. Just like in
Algorithm~\ref{algo:modified-selection} we distinguish two cases:
\begin{enumerate}
\item \label{case:1}
	The step size $\mu^{(t)}$ satisfies $1-\eta \leq \mu^{(t)} / \mu^* \leq 1+\eta$:\\
	We write the gain in iteration $t$ in the form
	$\big( 2 \mu^{(t)}/\mu^* - (\mu^{(t)}/\mu^*)^2 \big) \cdot \tilde g^*$,
	see Figure~\ref{fig:parabola1}. Together with the
	strict increase of the objective function in iteration $t+1$ we get
	$g^{\text{2-step}} \geq (1-\eta^2) \cdot \tilde g^*$.
\item
	The step size $\mu^{(t)}$ satisfies $\mu^{(t)} / \mu^* \not\in [1-\eta, 1+\eta]$:\\
	By construction the planned ahead gain~\eqref{eq:gain2} is lower
	bounded by $\tilde g^*$ (see Section~\ref{sec:planning-ahead}).
	The planning-step assumes that the working set $B^{(t-1)}$ is
	selected in iteration $t+1$. However, another working set may
	actually be chosen. Because Algorithm~\ref{algo:modified-selection}
	uses the SMO gain $g_B(\alpha^{(t)})$ for working set selection
	the gain may only improve due to the choice of
	$B^{(t+1)} \not= B^{(t-1)}$. Therefore $g^{\text{2-step}}$ is even
	lower bounded by $\tilde g^*$. With $1 - \eta^2 \leq 1$ the desired
	bound follows.
\end{enumerate}
\end{proof}

The first case seems to complicate things unnecessarily. Further, it
reduces the progress by a factor of $1 - \eta^2$. We could simply skip
the second if-condition in Algorithm~\ref{algo:modified-selection} and
in all cases turn to the else-part. From a purely mathematical point of
view this is clearly true. However, the usage of the exact gain
$g_B(\alpha)$ instead of $\tilde g_B(\alpha)$ is an unfavorable choice
for working set selection in practice. For performance reasons we want
to allow the algorithm to use the working set selection objective
$\tilde g_B(\alpha)$ as often as possible. Thus we have to cover
case~\ref{case:1} in Lemma~\ref{lemma:gain-rate},~too.

\begin{theorem}
Let $\alpha^{(t)}$ denote the sequence of feasible points produced by
the PA-SMO algorithm starting from $\alpha^{(0)}$ and working at perfect
accuracy $\varepsilon = 0$.
Then the algorithm either stops in finite time with an optimal solution
or produces an infinite sequence with
$\lim \limits_{t\to\infty}f(\alpha^{(t)}) = f^*$.
\end{theorem}
\begin{proof}
Because the algorithm checks the exact KKT conditions the finite
stopping case is trivial. For the infinite case we distinguish two
cases.
If the sequence contains only finitely many steps which are planning
ahead then there exists $t_0 > 0$ such that in all iterations $t > t_0$
the algorithm coincides with Algorithm~\ref{algo:SMO} and the
convergence proof given in~\cite{fan:2005} holds. Otherwise there
exists an infinite sequence $(t_n)_{n \in \N}$ of planning steps. Now
there are at least two possibilities to apply the above results.
An easy one is as follows: From Lemma~\ref{lemma:gain-rate} we obtain a
constant $c = 1 - \eta^2$ such that Theorem~\ref{theorem:convergence}
implies the desired property.
Alternatively we can argue that the double-step gain is non-negative by
Lemma~\ref{lemma:gain-rate}. Algorithm~\ref{algo:predictor} ensures that
the SMO steps in iterations $t_n - 1$, $n \in \N$ just before the
planning-ahead steps are free.
Then we can apply Corollary~\ref{corollary:convergence}.
However, the second variant of the proof does not hold if we replace
Algorithm~\ref{algo:predictor} by Algorithm~\ref{algo:modified-step}.
\end{proof}

As already noted above Theorem~\ref{theorem:convergence} and
Corollary~\ref{corollary:convergence} resolve the separate handling of
different cases by the algorithm in a general manner. In the case
of an infinite sequence of planning-ahead steps the proof does not
consider the other iterations at all. This technique is similar to the
convergence proof for the Hybrid Maximum-Gain second order algorithm
presented in \cite{glasmachers:2006} which needs to cover different
cases to ensure convergence, too.

\section{Experiments}

The main emphasis of the experiments is to compare the PA-SMO algorithm
with the standard (greedy) SMO algorithm. The most recent LIBSVM
version~2.84 implements Algorithm~\ref{algo:SMO}. For comparison, we
implemented the modifications described in
Algorithm~\ref{algo:modified-selection} and
Algorithm~\ref{algo:predictor} directly into LIBSVM.

Note that in the first iteration starting from
$\alpha^{(0)} = (0, \dots, 0)^T$ the components $y_i = \pm 1$ of the
gradient $\nabla f(\alpha^{(0)}) = y$ take only two possible values.
The absolute values of these components are equal and they all point
into the box. Therefore the working set selection algorithm could select
any $i^{(1)} \in \Iup(\alpha^{(0)})$ as the first index, because the
gradient components of all indices are maximal. Thus, there is a freedom
of choice for the first iteration. LIBSVM arbitrarily chooses
$i^{(1)} = \max(\Iup(\alpha^{(0)}))$. Of course, this choice influences
the path taken by the optimization. Experiments indicate that this
choice can have a significant impact on the number of iterations and the
runtime of the algorithm. Now, on a fixed dataset, an algorithm may
appear to be superior to another one just because it is lucky to profit
more from the asymmetry than the competitor.
To reduce random effects, we created 100 random permutations of each
dataset. All measurements reported are mean values over these 100
permutations. Because the permutations were drawn i.i.d.\ we can apply
standard significance tests to our measurements.

We collected a set of 22 datasets for the performance comparison.
For the 13 benchmark datasets from \cite{raetsch:2001} we merged
training and test sets. The artificial chess-board
problem~\cite{glasmachers:2005} was considered because it corresponds
to quadratic programs which are very difficult to solve for SMO-type
decomposition algorithms. Because this problem is described by a known
distribution, we are in the position to sample datasets of any size from
it. We arbitrarily fixed three datasets consisting of $1,000$, $10,000$,
and $100,000$ examples.
Six more datasets were taken from the UCI benchmark
repository~\cite{uci:98}: The datasets \ds{connect-4},
\ds{king-rook-vs-king}, and \ds{tic-tac-toe} are extracted from games,
while \ds{ionosphere}, \ds{spambase}, and \ds{internet-ads} stem from
real world applications.

In all experiments we use the Gaussian kernel
\begin{align*}
	k(x_i, x_j) = \exp(-\gamma \, \|x_i - x_j\|^2)
\end{align*}
with the single kernel parameter $\gamma > 0$. The complexity control
parameter $C$ and the kernel parameter $\gamma$ were selected with grid
search on the cross-validation error to ensure that the parameters are
in a regime where the resulting classifiers generalize reasonably well,
see Table~\ref{tab:setting}.
\begin{table}
\begin{center}
\begin{tabular}{|l|r|rr|rr|}
\hline
\multicolumn{1}{|c}{dataset}	& \multicolumn{1}{|c}{$\ell$}	& \multicolumn{1}{|c}{$C$}	& \multicolumn{1}{c}{$\gamma$}	& \multicolumn{1}{|c}{SV}	& \multicolumn{1}{c|}{BSV}	\\
\hline
\ds{banana}						& \tabfont 5,300				& \tabfont 100				& \tabfont 0.25					& \tabfont 1,223			& \tabfont 1,199			\\
\ds{breast-cancer}				& \tabfont 277					& \tabfont 0.6				& \tabfont 0.1					& \tabfont 178				& \tabfont 131				\\
\ds{diabetis}					& \tabfont 768					& \tabfont 0.5				& \tabfont 0.05					& \tabfont 445				& \tabfont 414				\\
\ds{flare-solar}				& \tabfont 1,066				& \tabfont 1.5				& \tabfont 0.1					& \tabfont 744				& \tabfont 709				\\
\ds{german}						& \tabfont 1,000				& \tabfont 1				& \tabfont 0.05					& \tabfont 620				& \tabfont 426				\\
\ds{heart}						& \tabfont 270					& \tabfont 1				& \tabfont 0.005				& \tabfont 158				& \tabfont 149				\\
\ds{image}						& \tabfont 2,310				& \tabfont 100				& \tabfont 0.1					& \tabfont 301				& \tabfont 84				\\
\ds{ringnorm}					& \tabfont 7,400				& \tabfont 2				& \tabfont 0.1					& \tabfont 625				& \tabfont 86				\\
\ds{splice}						& \tabfont 3,175				& \tabfont 10				& \tabfont 0.01					& \tabfont 1,426			& \tabfont 7				\\
\ds{thyroid}					& \tabfont 215					& \tabfont 500				& \tabfont 0.05					& \tabfont 17				& \tabfont 3				\\
\ds{titanic}					& \tabfont 2,201				& \tabfont 1,000			& \tabfont 0.1					& \tabfont 934				& \tabfont 915				\\
\ds{twonorm}					& \tabfont 7,400				& \tabfont 0.5				& \tabfont 0.02					& \tabfont 734				& \tabfont 662				\\
\ds{waveform}					& \tabfont 5,000				& \tabfont 1				& \tabfont 0.05					& \tabfont 1,262			& \tabfont 980				\\
\hline
\ds{chess-board-1000}			& \tabfont 1,000				& \tabfont 1,000,000		& \tabfont 0.5					& \tabfont 41				& \tabfont 3				\\
\ds{chess-board-10000}			& \tabfont 10,000				& \tabfont 1,000,000		& \tabfont 0.5					& \tabfont 129				& \tabfont 84				\\
\ds{chess-board-100000}			& \tabfont 100,000				& \tabfont 1,000,000		& \tabfont 0.5					& \tabfont 556				& \tabfont 504				\\
\hline
\ds{connect-4}					& \tabfont 61,108				& \tabfont 4.5				& \tabfont 0.2					& \tabfont 13,485			& \tabfont 5,994			\\
\ds{king-rook-vs-king}			& \tabfont 28,056				& \tabfont 10				& \tabfont 0.5					& \tabfont 5,815			& \tabfont 206				\\
\ds{tic-tac-toe}				& \tabfont 958					& \tabfont 200				& \tabfont 0.02					& \tabfont 104				& \tabfont 0				\\
\ds{internet-ads}				& \tabfont 2,358				& \tabfont 10				& \tabfont 0.03					& \tabfont 1,350			& \tabfont 6				\\
\ds{ionosphere}					& \tabfont 351					& \tabfont 3				& \tabfont 0.4					& \tabfont 190				& \tabfont 8				\\
\ds{spam-database}				& \tabfont 4,601				& \tabfont 10				& \tabfont 0.005				& \tabfont 1,982			& \tabfont 583				\\
\hline
\end{tabular}
\end{center}
\caption{\label{tab:setting}Datasets used for the comparison. The
	dataset size, the regularization parameter $C$ and the kernel
	parameter $\gamma$ are given. The last two columns list the
	resulting total number of support vectors and the number of
	bounded support vectors. Due to the finite accuracy of the
	solutions these mean values are not always integers. For clarity
	we provide rounded values.}
\end{table}
All experiments were carried out on a Xeon 3~GHz CPU running Fedora
Linux.

\subsection{Results}

We performed 100 runs (corresponding to the 100 permutations) per
dataset for both algorithms and measured the runtime and the number of
iterations. The results are summarized in Table~\ref{tab:results}.

\begin{table}
\begin{center}
\begin{tabular}{|l|ccc|ccc|}
\hline
\multicolumn{1}{|c}{dataset}	& \multicolumn{3}{|c}{time}	& \multicolumn{3}{|c|}{iterations}	\\
								& SMO & & PA-SMO			& SMO & & PA-SMO					\\
\hline
\ds{banana} & 2.07 &  & 2.08 & 23295 & $>$ & 19721 \\
\ds{breast-cancer} & 0.02 &  & 0.02 & 313 & $>$ & 292 \\
\ds{diabetis} & 0.08 &  & 0.10 & 361 & $>$ & 358 \\
\ds{flare-solar} & 0.18 &  & 0.19 & 792 & $>$ & 744 \\
\ds{german} & 0.20 & $>$ & 0.19 & 908 & $>$ & 879 \\
\ds{heart} & 0.02 &  & 0.02 & 113 &  & 112 \\
\ds{image} & 0.45 &  & 0.46 & 6553 & $>$ & 6359 \\
\ds{ringnorm} & 2.41 &  & 2.27 & 1569 & $>$ & 1537 \\
\ds{splice} & 4.04 & $>$ & 3.92 & 6643 & $>$ & 5854 \\
\ds{thyroid} & 0.02 &  & 0.01 & 744 & $>$ & 667 \\
\ds{titanic} & 0.54 & $>$ & 0.47 & 3375 & $>$ & 1653 \\
\ds{twonorm} & 2.67 &  & 2.65 & 641 &  & 642 \\
\ds{waveform} & 3.03 &  & 2.99 & 1610 & $>$ & 1539 \\
\hline
\ds{chess-board-1000} & 3.86 & $>$ & 2.98 & 1883310 & $>$ & 1186963 \\
\ds{chess-board-10000} & 76.72 &  & 75.36 & 32130476 & $>$ & 24997371 \\
\ds{chess-board-100000} & 475.37 & $>$ & 428.18 & 145364030 & $>$ & 105199379 \\
\hline
\ds{connect4-0.2} & 1268.04 &  & 1243.56 & 82076 & $>$ & 77690 \\
\ds{king-rook-vs-king} & 272.80 &  & 273.06 & 69410 & $>$ & 64067 \\
\ds{tic-tac-toe} & 0.10 &  & 0.10 & 8321 & $>$ & 7786 \\
\ds{internet-ads} & 2.38 &  & 2.31 & 2785 & $>$ & 2750 \\
\ds{ionosphere} & 0.03 &  & 0.04 & 411 & $>$ & 408 \\
\ds{spambase} & 8.36 &  & 8.36 & 9641 & $>$ & 9171 \\
\hline
\end{tabular}
\end{center}
\caption{\label{tab:results}Comparison of standard SMO
(Algorithm~\ref{algo:SMO}) and planning-ahead SMO
(Algorithm~\ref{algo:PA-SMO}). Mean time in seconds and number of
iterations are listed. The ``$>$'' sign indicates that the left value
is statistically significantly larger than the right value (paired
Wilcoxon rank rum test, $p=0.05$ over 100 permutations of the datasets).
The left value is in no case significantly smaller than the right one.}
\end{table}

There is a clear trend in these results. For some datasets the PA-SMO
algorithm significantly outperforms the SMO algorithm, while for other
datasets there is no significant difference. Most important, PA-SMO
performs {\em in no case} worse than standard SMO.

The number of iterations is significantly reduced in nearly all cases.
This result is not surprising. It basically means that the algorithm
works as expected. However, early iterations working on the whole
problem take much longer than late iterations after shrinking has more
or less identified the interesting variables. Therefore it is natural
that the number of iterations is only a weak indicator for the runtime.
The runtime of the PA-SMO algorithm is usually slightly reduced in the
mean. This difference is significant in 5 cases. However, the striking
argument for the algorithm is that it never performs worse than the
standard SMO algorithm.

Although both algorithms use the same stopping condition the dual
objective values achieved slightly varies. A careful check of these
values reveals that the PA-SMO algorithm consistently achieves better
solutions (paired Wilcoxon rank sum test, $p = 0.05$) for all datasets
but \ds{chess-board-100000}.
Thus, the speed up is not a trivial effect of reduced solution quality.
The tests reveal that the contrary is the case, that is, the new
algorithm outputs better solutions in less time.

\subsection{Influence of Planning-Ahead vs.\ Working Set Selection}

It is interesting to look a little bit behind the scenes. Recall that we
changed two parts of the SMO algorithm. The truncated Newton step was
replaced by the planning-ahead Algorithm~\ref{algo:predictor} and the
working set selection was modified accordingly by
Algorithm~\ref{algo:modified-selection}. It is possible to use the
second modification without the first one, but hardly vice versa.
Therefore, we ran the SMO algorithm with the modified working set
selection but without planning ahead to get a grip on the influence of
these changes on the overall performance. That is, we made sure that the
algorithm selects the working set used two iterations ago if it is a
feasible direction and maximizes the Newton step gain $\tilde g$.
While the results of the comparison to standard SMO were completely
ambiguous, the PA-SMO algorithm turned out to be clearly superior.
Thus, the reason for the speed up of PA-SMO is not the changed working
set selection, but planning-ahead.

\subsection{Planning-Ahead Step Sizes}

To understand how planning-ahead is really used by the algorithm we
measured the quantity $\mu / \mu^* - 1$, that is, the size of the
planning-ahead step relative to the Newton step. For free SMO steps
this quantity is always~$0$, for larger steps it is positive, for
smaller steps negative, and for steps in the opposite direction it
is even smaller than~$-1$. We present some representative histograms
in Figure~\ref{fig:histograms}.
These histograms reveal that most planning-steps are only slightly
increased compared to the Newton step size, but there are cases where
the algorithm chooses a step which is enlarged by a factor of several
thousands. However, very few steps are reduced or even reversed, if any.

\begin{figure}[h]
\begin{center}
\psfrag{0}[c][l]{$-100$}
\psfrag{1}[c][l]{$-3.16$}
\psfrag{2}[c][l]{$0$}
\psfrag{3}[c][l]{$3.16$}
\psfrag{4}[c][l]{$100$}
\psfrag{5}[c][l]{$31623$}
\psfrag{100}[r][l]{$10^0$}
\psfrag{101}[r][l]{$10^1$}
\psfrag{102}[r][l]{$10^2$}
\psfrag{103}[r][l]{$10^3$}
\psfrag{104}[r][l]{$10^4$}
\psfrag{105}[r][l]{$10^5$}
\psfrag{106}[r][l]{$10^6$}
\psfrag{107}[r][l]{$10^7$}
\psfrag{108}[r][l]{$10^8$}
\psfrag{banana}[c][l]{\ds{banana}}
\psfrag{heart}[c][l]{\ds{heart}}
\psfrag{image}[c][l]{\ds{image}}
\psfrag{twonorm}[c][l]{\ds{twonorm}}
\psfrag{chess1000}[c][l]{\ds{chess-board-1000}}
\psfrag{spambase}[c][l]{\ds{spambase}}
\includegraphics[width=0.95\textwidth]{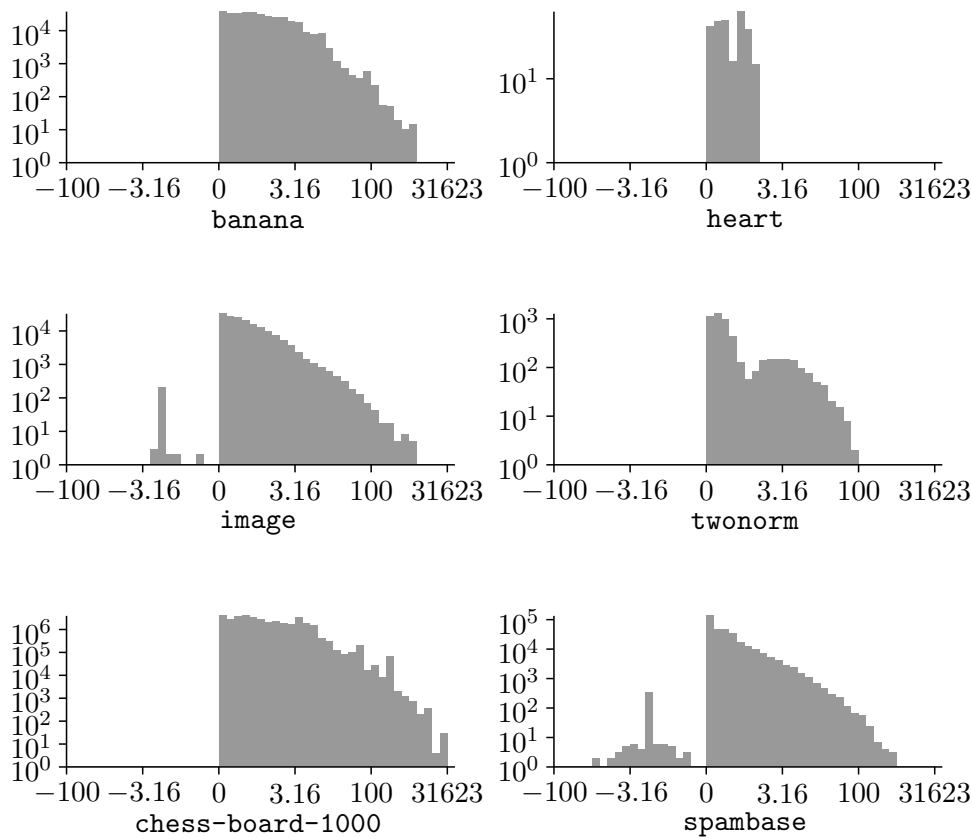}
\caption{\label{fig:histograms}
Histograms (number of iterations) of the planning-step size $\mu$
divided by the Newton step size $\mu^*$, minus $1$.
On both axes a logarithmic scale is used to increase the resolution for
small values (to achieve this effect for the $x$-axis we used the
parameterization $t \mapsto \sign(t) \cdot (10^{t^2/2} - 1)$ which
is symmetric around the Newton step size corresponding to the origin of
$t = \mu / \mu^* - 1$, with a high resolution around this point).
The rightmost bin counts all steps which exceed the scale, which is
actually the case for the \ds{chess-board-1000} dataset.}
\end{center}
\end{figure}

Obviously the step size histograms are far from symmetric. Therefore it
is natural to ask whether a very simple increase of the Newton step size
can be a good strategy. By heretically using
\begin{align*}
	\mu^{(t)} = \max \left\{ \min \left\{ 1.1 \cdot \frac{l_t}{Q_{tt}}, \tilde U_t \right\}, \tilde L_t \right\}
\end{align*}
instead of equation~\eqref{eq:clipped-newton} we still achieve
$1 - 0.1^2 = 99\%$ of the SMO gain in each iteration
(see Figure~\ref{fig:parabola1}) and avoid the drawback of the more
complicated computations involved when planning-ahead. Further, this
strategy can be implemented into an existing SMO solver in just a few
seconds. Experiments indicate that it is surprisingly successful, no
matter if the original working set selection or
Algorithm~\ref{algo:modified-selection} is used.
For most simple problems it performs as good as the much more refined
PA-SMO strategy. However, for the extremely difficult \ds{chess-board}
problem this strategy performs significantly worse.

\subsection{Multiple Planning-Ahead}

We now turn to the variant of the PA-SMO algorithm which uses more than
one recent working set for planning-ahead. This variant, as explained at
the end of section~\ref{sec:algorithms}, plans ahead with multiple
candidate working sets. Further, these working sets are additional
candidates for the working set selection. We can expect that the number
of iterations decreases the more working sets are used this way.
However, the computations per iteration of course increase, such that
too many working sets will slow the entire algorithm down. Thus, there
is a trade-off between the number of iterations and the time needed per
iterations. Now the interesting question is whether there is a uniform
best number of working sets for all problems.

\begin{figure}[h]
\begin{center}
\psfrag{1}[c][l]{$1$}
\psfrag{2}[c][l]{$2$}
\psfrag{3}[c][l]{$3$}
\psfrag{5}[c][l]{$5$}
\psfrag{10}[c][l]{$10$}
\psfrag{20}[c][l]{$20$}
\psfrag{80}[r][l]{$80\,\%$}
\psfrag{90}[r][l]{$90\,\%$}
\psfrag{100}[r][l]{$100\,\%$}
\psfrag{110}[r][l]{$110\,\%$}
\psfrag{120}[r][l]{$120\,\%$}
\psfrag{banana}[l][l]{\dss{banana}}
\psfrag{flare-solar}[l][l]{\dss{flare-solar}}
\psfrag{german}[l][l]{\dss{german}}
\psfrag{image}[l][l]{\dss{image}}
\psfrag{ringnorm}[l][l]{\dss{ringnorm}}
\psfrag{splice}[l][l]{\dss{splice}}
\psfrag{titanic}[l][l]{\dss{titanic}}
\psfrag{twonorm}[l][l]{\dss{twonorm}}
\psfrag{waveform}[l][l]{\dss{waveform}}
\psfrag{chess-board-1000}[l][l]{\dss{chess-board-1000}}
\psfrag{chess-board-10000}[l][l]{\dss{chess-board-10000}}
\psfrag{chess-board-100000}[l][l]{\dss{chess-board-100000}}
\psfrag{connect-4}[l][l]{\dss{connect-4}}
\psfrag{king-rook-vs-king}[l][l]{\dss{king-rook-vs-king}}
\psfrag{internet-ads}[l][l]{\dss{internet-ads}}
\psfrag{spambase}[l][l]{\dss{spambase}}
\includegraphics[width=\textwidth]{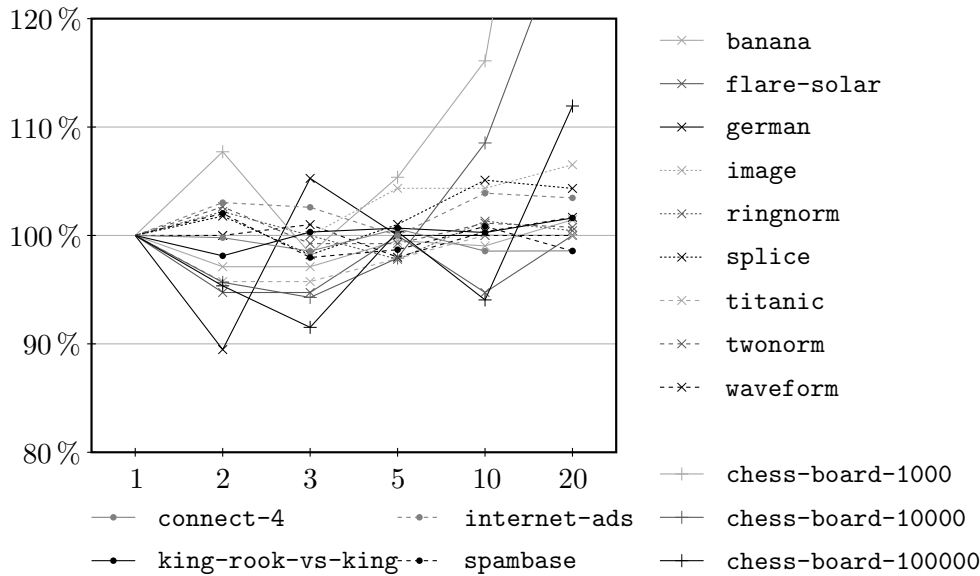}
\caption{\label{fig:multiple-ws}
The curves show the performance (mean runtime) of the PA-SMO algorithm
with $1$, $2$, $3$, $5$, $10$, and $20$ most recent working sets.
All results are normalized with the runtime of the standard variant
(only the most recent working set is used for planning ahead).
Only datasets with runtimes above $100$~ms are plotted, because
otherwise the deviation due to the measurement accuracy shadows the
effect.}
\end{center}
\end{figure}

We performed experiments with the $2$, $3$, $5$, $10$, and $20$ most
recent working sets. It turned out that the strategies considering the
most recent two or three working sets perform comparable to standard
PA-SMO, and even slightly better. For $5$, and more drastically, for
$10$ or $20$ working set evaluations the performance drops, see
Figure~\ref{fig:multiple-ws}.
This result makes clear that we do not lose much when completely
ignoring the multiple working set selection strategy, and at the same
time we stay at the safe side. Therefore it seems reasonable to stick
to the standard form of the PA-SMO algorithm. On the other hand we can
get another small improvement if the two or three most recent working
sets are considered.

\section{Conclusion}

We presented the planning-ahead SMO algorithm (PA-SMO), which is a
simple yet powerful improvement of SMO.
At a first glance it is surprising that the truncated Newton step used
in all existing variants of the SMO algorithm can be outperformed. This
becomes clear from the greedy character of decomposition iterations.
The experimental evaluation clearly shows the benefits of the new
algorithm. As we never observed a decrease in performance, we recommend
PA-SMO as the default algorithm for SVM training.
PA-SMO is easy to implement based on existing SMO solvers.
Due to the guaranteed convergence of the algorithm to an optimal
solution the method is widely applicable. Further, the convergence
proof introduces a general technique to address the convergence of
hybrid algorithms.

\begin{algorithm}
	\SetKwRepeat{InfiniteLoop}{do}{loop}
	\KwIn{feasible initial point $\alpha^{(0)}$, accuracy $\varepsilon \geq 0$, $\eta \in (0, 1)$}
	$G^{(0)} \leftarrow \nabla f(\alpha^{(0)}) = y - K \alpha^{(0)}$, $p \leftarrow \text{true}$, $t \leftarrow 1$

	\InfiniteLoop{}
	{
		\eIf{p = \text{true}}
		{
			$i^{(t)} \leftarrow \argmax\{G_n^{(t-1)} \enspace|\enspace n \in \Iup(\alpha^{(t-1)})\}$\\
			$j^{(t)} \leftarrow \argmax\{\tilde g_{(i^{(t)}, n)}(\alpha^{(t-1)}) \enspace|\enspace n \in \Idown(\alpha^{(t-1)}) \setminus \{i\} \}$\\
			$B^{(t)} \leftarrow (i^{(t)}, j^{(t)})$
		}
		{
			\eIf{$1-\eta \leq \mu^{(t-1)} / \mu^* \leq 1+\eta$}
			{
				$i^{(t)} \leftarrow \argmax\{G_n^{(t-1)} \enspace|\enspace n \in \Iup(\alpha^{(t-1)})\}$\\
				$j^{(t)} \leftarrow \argmax\{\tilde g_{(i^{(t)}, n)}(\alpha^{(t-1)}) \enspace|\enspace n \in \Idown(\alpha^{(t-1)}) \setminus \{i\} \}$\\
				$B^{(t)} \leftarrow (i^{(t)}, j^{(t)})$\\
				\textbf{if} $\tilde g_{B^{(t-2)}}(\alpha^{(t-1)}) > \tilde g_{B^{(t)}}(\alpha^{(t-1)})$ \textbf{then} $B^{(t)} \leftarrow B^{(t-2)}$
			}
			{
				$i^{(t)} \leftarrow \argmax\{G_n^{(t-1)} \enspace|\enspace n \in \Iup(\alpha^{(t-1)})\}$\\
				$j^{(t)} \leftarrow \argmax\{g_{(i^{(t)}, n)}(\alpha^{(t-1)}) \enspace|\enspace n \in \Idown(\alpha^{(t-1)}) \setminus \{i\} \}$\\
				$B^{(t)} \leftarrow (i^{(t)}, j^{(t)})$\\
				\textbf{if} $g_{B^{(t-2)}}(\alpha^{(t-1)}) > g_{B^{(t)}}(\alpha^{(t-1)})$ \textbf{then} $B^{(t)} \leftarrow B^{(t-2)}$
			}
		}
		$\mu^* \leftarrow l_t / Q_{tt} = (v_{B^{(t)}}^T G^{(t-1)}) / (v_{B^{(t)}}^T K v_{B^{(t)}})$\\
		\eIf{previous iteration performed a free SMO step}
		{
			Compute the planning-ahead step size $\mu = \frac{Q_{22} w_1 - Q_{12} w_2}{\det(Q)}$ (eq.~\eqref{eq:planning-step})\\
			$\qquad$ assuming $B^{(t-1)}$ as the next working set\\
			\eIf{the current or the planned step ends at the box boundary}
			{
				$\mu^{(t)} \leftarrow \max \left\{ \min \left\{ \mu^*, \tilde U_t \right\}, \tilde L_t \right\}$ (eq.~\eqref{eq:clipped-newton})\\
				$p \leftarrow \text{false}$
			}
			{
				$\mu^{(t)} \leftarrow \mu$\\
				$p \leftarrow \text{true}$
			}
		}
		{
			$\mu^{(t)} \leftarrow \max \left\{ \min \left\{ \mu^*, \tilde U_t \right\}, \tilde L_t \right\}$ (eq.~\eqref{eq:clipped-newton})\\
			$p \leftarrow \text{false}$
		}
		$\alpha^{(t)} \leftarrow \alpha^{(t-1)} + \mu^{(t)} \cdot v_{B^{(t)}}$\\
		$G^{(t)} \leftarrow G^{(t-1)} - \mu^{(t)} K v_{B^{(t)}}$\\
		\textbf{stop if} $\left( \max \big\{ G^{(t)}_i\ \,\big|\, i \in \Iup(\alpha^{(t)}) \big\} - \min \big\{ G^{(t)}_j \,\big|\, j \in \Idown(\alpha^{(t)}) \big\} \right) \leq \varepsilon$\\
		$t \leftarrow t + 1$
	}
	\caption{The complete PA-SMO Algorithm}
	\label{algo:PA-SMO}
\end{algorithm}

\vskip 0.2in
\bibliographystyle{plain}
\bibliography{SmoStepSize}

\begin{thebibliography}{10}

\bibitem{chen:2006}
P.-H. Chen, R.-E. Fan, and C.-J. Lin.
\newblock A {S}tudy on {SMO}-type {D}ecomposition {M}ethods for {S}upport
  {V}ector {M}achines.
\newblock {\em IEEE Transactions on Neural Networks}, 17:893--908, 2006.

\bibitem{fan:2005}
R.-E. Fan, P.-H. Chen, and C.-J. Lin.
\newblock Working {S}et {S}election using the {S}econd {O}rder {I}nformation
  for {T}raining {S}upport {V}ector {M}achines.
\newblock {\em Journal of Machine Learning Research}, 6:1889--1918, 2005.

\bibitem{glasmachers:esann08}
T.~Glasmachers.
\newblock On related violating pairs for working set selection in {SMO}
  algorithms.
\newblock Submitted to European Symposium on Artificial Neural Networks
  (ESANN), 2008.

\bibitem{glasmachers:2005}
T.~Glasmachers and C.~Igel.
\newblock Gradient-based {A}daptation of {G}eneral {G}aussian {K}ernels.
\newblock {\em Neural Computation}, 17(10):2099--2105, 2005.

\bibitem{glasmachers:2006}
T.~Glasmachers and C.~Igel.
\newblock Maximum {G}ain {W}orking {S}et {S}election for {SVM}s.
\newblock {\em Journal of Machine Learning Research}, 7:1437--1466, 2006.

\bibitem{hush:2003}
D.~Hush and C.~Scovel.
\newblock Polynomial-time {D}ecomposition {A}lgorithms for {S}upport {V}ector
  {M}achines.
\newblock {\em Machine Learning}, 51:51--71, 2003.

\bibitem{joachims:99}
T.~Joachims.
\newblock Making {L}arge-{S}cale {SVM} {L}earning {P}ractical.
\newblock In B.~Sch\"olkopf, C.~Burges, and A.~Smola, editors, {\em Advances in
  Kernel Methods -- {S}upport {V}ector {L}earning}, chapter~11, pages 169--184.
  MIT Press, 1999.

\bibitem{keerthi-gilbert:2002}
S.~S. Keerthi and E.~G. Gilbert.
\newblock Convergence of a {G}eneralized {SMO} {A}lgorithm for {SVM}
  {C}lassifier {D}esign.
\newblock {\em Machine Learning}, 46:351--360, 2002.

\bibitem{list:2007}
N.~List, D.~Hush, C.~Scovel, and I.~Steinwart.
\newblock Gaps in {S}upport {V}ector {O}ptimization.
\newblock In {\em Learning Theory, 20th Annual Conference on Learning Theory,
  COLT 2007. Lecture Notes in Computer Science}, volume 4539, pages 336--348,
  2007.

\bibitem{list:2004}
N.~List and H.~U. Simon.
\newblock A {G}eneral {C}onvergence {T}heorem for the {D}ecomposition {M}ethod.
\newblock In J.~Shawe-Taylor and Y.~Singer, editors, {\em Proceedings of the
  17th Annual Conference on Learning Theory, COLT 2004}, volume 3120 of {\em
  LNCS}, pages 363--377. Springer-Verlag, 2004.

\bibitem{uci:98}
D.~J. Newman, S.~Hettich, C.~L. Blake, and C.~J. Merz.
\newblock {UCI} {R}epository of machine learning databases, 1998.
\newblock http://www.ics.uci.edu/~mlearn/MLRepository.html.

\bibitem{osuna:97}
E.~Osuna, R.~Freund, and F.~Girosi.
\newblock An {I}mproved {T}raining {A}lgorithm for {S}upport {V}ector
  {M}achines.
\newblock In J.~Principe, L.~Giles, N.~Morgan, and E.~Wilson, editors, {\em
  Neural Networks for Signal Processing VII}, pages 276--285. IEEE Press, 1997.

\bibitem{platt:99}
J.~Platt.
\newblock Fast {T}raining of {S}upport {V}ector {M}achines using {S}equential
  {M}inimal {O}ptimization.
\newblock In B.~Sch\"olkopf, C.~J.~C. Burges, and A.~J. Smola, editors, {\em
  Advances in Kernel Methods - Support Vector Learning}, chapter~12, pages
  185--208. MIT Press, 1999.

\bibitem{raetsch:2001}
G.~R{\"a}tsch, T.~Onoda, and K.-R. M\"uller.
\newblock Soft {M}argins for {AdaBoost}.
\newblock {\em Machine Learning}, 42(3):287--320, 2001.

\bibitem{takahashi:2005}
N.~Takahashi and T.~Nishi.
\newblock Rigorous {P}roof of {T}ermination of {SMO} {A}lgorithm for {S}upport
  {V}ector {M}achines.
\newblock {\em IEEE Transaction on Neural Networks}, 16(3):774--776, 2005.

\end{thebibliography}

\end{document}